\documentclass[twoside]{article}
\usepackage[accepted]{aistats2024}
\usepackage[T1]{fontenc}

\usepackage{algorithm}
\usepackage{algorithmic}
\usepackage{amssymb}
\usepackage{amsthm}
\usepackage{bm}
\usepackage{booktabs}
\usepackage{caption}
\usepackage{enumitem}
\usepackage{graphicx}

\PassOptionsToPackage{hyphens}{url}
\usepackage{hyperref}

\usepackage{makecell}
\usepackage{mathtools}
\usepackage{microtype}
\usepackage{multirow}
\usepackage{natbib}
\usepackage{nicefrac}
\usepackage{relsize}
\usepackage{subcaption}
\usepackage{tabularx}
\usepackage{xcolor}

\makeatletter
\newcommand{\subalign}[1]{%
  \vcenter{%
    \Let@ \restore@math@cr \default@tag
    \baselineskip\fontdimen10 \scriptfont\tw@
    \advance\baselineskip\fontdimen12 \scriptfont\tw@
    \lineskip\thr@@\fontdimen8 \scriptfont\thr@@
    \lineskiplimit\lineskip
    \ialign{\hfil$\m@th\scriptstyle##$&$\m@th\scriptstyle{}##$\hfil\crcr
      #1\crcr
    }%
  }%
}
\makeatother

\newtheorem{proposition}{Proposition}

\DeclareMathOperator{\argmin}{argmin}
\DeclareMathOperator{\diag}{diag}

\allowdisplaybreaks
\newcommand{\justify}[1]{{#1\parfillskip 0pt\par}}
\newcommand{\newtilde}{\mbox{\raise.17ex\hbox{$\scriptstyle\sim$}}}
\renewcommand{\paragraph}[1]{\textbf{#1}\hspace{10pt}}

\begin{document}

\twocolumn[
    \aistatstitle{Adaptive Experiment Design with Synthetic Controls}
    \aistatsauthor{Alihan H\"uy\"uk \And Zhaozhi Qian \And Mihaela van der Schaar}
    \aistatsaddress{University of Cambridge \And University of Cambridge \And University of Cambridge}]

\begin{abstract}
    \vspace{-9pt}
    Clinical trials are typically run in order to understand the effects of a new treatment on a given population of patients. However, patients in large populations rarely respond the same way to the same treatment. This heterogeneity in patient responses necessitates trials that investigate effects on \textit{multiple subpopulations}---especially when a treatment has marginal or no benefit for the overall population but might have significant benefit for a particular subpopulation. Motivated by this need, we propose \textsc{Syntax}, an exploratory trial design that identifies subpopulations with positive treatment effect among many subpopulations. \textsc{Syntax} is sample efficient as it (i) recruits and allocates patients \textit{adaptively} and (ii) estimates treatment effects by forming \textit{synthetic controls} for each subpopulation that combines control samples from other subpopulations. We validate the performance of \textsc{Syntax} and provide insights into when it might have an advantage over conventional trial designs through experiments.
\end{abstract}

%%%
\vspace{-12pt}
\section{INTRODUCTION}
\vspace{-3pt}

\begin{figure}
    \vspace{-3pt}%
    \includegraphics[width=\linewidth]{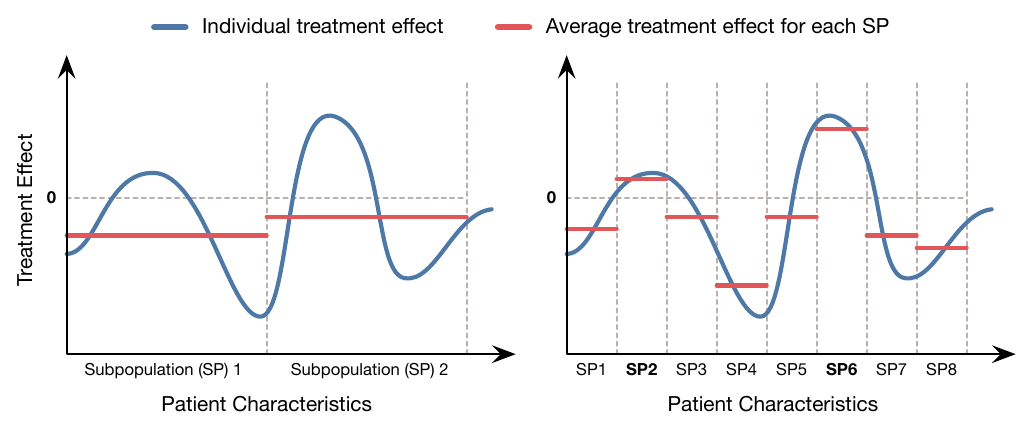}%
    \vspace{2.4pt}%

    \small%
    \resizebox{\linewidth}{!}{
        \begin{tabular}{@{}p{150pt}@{\hspace{12pt}}p{150pt}@{}}
            \toprule
            \hfill \textbf{Clinical Trial A} & \textbf{Clinical Trial B} \\
            \midrule
            \textit{SPs w/ positive effect:} \hfill None & SP2 \& SP6 \\
            \textit{\# of patients per SP:} \hfill 50 samples & \newtilde 12 samples \hfill \textit{(out of 100 samples)} \\
            \bottomrule
        \end{tabular}}%
    \vspace{-1.2pt}%
    \caption{\textbf{Tradeoff between individual benefit and cost.} Consider two clinical trials that are both designed to confirm the effectiveness of a new treatment for one or multiple subpopulations. While Trial~A investigates only two candidate subpopulations, Trial~B investigates eight. As a result, Trial~B has the potential to succeed for two subpopulations (SP2 \& SP6) while Trial~A is likely to fail for all. However, Trial~B needs to allocate fewer samples to each subpopulation, which makes confirming positive effects more challenging. We propose \textsc{Syntax} as an exploratory pilot study that finds good subpopulations to target (such as SP2 \& SP6) ahead of a confirmatory trial.}%
    \label{fig:tradeoff}%
    \vspace{-\baselineskip}%
\end{figure}

\justify{Randomized controlled trials (RCTs) remain an essential part of evidence-based medicine \citep{sackett1995need} despite their recognized shortcomings \citep{feinstein1997problems}. Many of these shortcomings are starting to be addressed by machine learning \citep{curth2023adaptively,huyuk2023when}; in this paper, we focus on a specific one: RCTs typically only consider \textit{average} treatment effects across a given population, yet patients in large populations rarely respond the same way to receiving the same treatment. Differences in genetics, environments, and clinical backgrounds among patients lead to differences in the benefit they receive from a treatment as well. This heterogeneity in patient responses necessitates clinical trial designs that investigate the treatment effect not only for the overall population but also for various \textit{subpopulations} within it \citep{chiu2018design}. Investigating multiple subpopulations at once becomes especially important when a treatment only has marginal benefit for the population as a whole while it might have significant benefit for a particular subpopulation \citep{moineddin2008identifying,lipkovich2017tutorial}. Declaring such a treatment to be ineffective after a clinical trial that ignores heterogeneity might result in the treatment being needlessly denied to the subpopulation that would have actually benefited from it.}

Running more ``complex'' trials that investigate larger numbers of subpopulations provides an opportunity to identify more personalized treatments for each subpopulation, rather than identifying one-size-fits-all treatments that are supposed to be beneficial for everyone \citep{wang2007statistics,lazar2016identifying}. However, such trials also come with increased costs as dividing the population into finer subpopulations naturally requires more patients to be recruited into the trial when the recruitment is done in a randomized manner \citep{whitley2004subgroup,ondra2016methods}. When determining how many subpopulations a trial should investigate, there is a clear tradeoff between the benefit provided to the individual and the cost of running the trial 
(Figure~\ref{fig:tradeoff}).
% (see Figure~\ref{fig:tradeoff} for a concrete example).

Motivated by this issue, we propose \textsc{Syntax} (\textit{Synthetic Adaptive eXploration}), an exploratory trial design for subpopulation selection.
Exploratory (Phase I \& II) trials are commonly run to establish a treatment's safety and to determine the correct dosage for the treatment. In contrast, confirmatory (Phase III) trials aim to validate the efficacy of a treatment with type I error control. \textsc{Syntax}, as an exploratory (pre-Phase III, similar to Phase I \& II) trial design, aims to find suitable populations for which validating efficacy would be easier in a subsequent confirmatory trial. Specifically, our objective is to identify subpopulations with positive treatment effect among a pre-determined set of candidate subpopulations with a limited budget of samples---that is a limited number of patients that can be recruited into the trial. While a conventional Phase II trial determines the ``right'' dosage for a Phase III trial, in a similar vein, \textsc{Syntax} determines the ``right'' populations to target in a Phase III trial.

\justify{Our trial design has two key characteristics that contribute to its efficiency in identifying subpopulations:}
\begin{enumerate}[label=(\roman*),nosep,parsep=\parskip,leftmargin=18pt]
    \item \textsc{Syntax} recruits patients and allocates them into control or treatment groups \textit{adaptively}, not randomly. This allows more samples to be allocated into subpopulations that have higher variability in their responses to the treatment.
    
    \item When estimating treatment effects, \textsc{Syntax} forms \textit{synthetic controls} for each subpopulation, by decomposing them into a linear combination of others, instead of relying on control samples just from the subpopulation itself. This allows information from control samples to be shared between different subpopulations and more samples to be allocated to the treatment group.
\end{enumerate}

An important difference bears emphasis here: Conventionally, synthetic controls replace the control group in a trial outright, they are sourced from \textit{offline datasets} collected ahead of the trial, and no new control samples are observed during the trial itself. In this work, we consider synthetic control in an \textit{online context} and investigate how they can still facilitate more efficient use of the control samples that are being collected as a trial continues. This setting is unique within the synthetic control literature, as it requires an experimenter to decide continually for each subpopulation, whether to (i) collect new real data as controls, or (ii)~rely on \mbox{synthetic} controls constructed from previously observed data for other subpopulations.

\paragraph{Contributions}
Our contributions are three-fold: \textit{First,} we outline a treatment effect estimator based on decomposing subpopulations into linear combinations of others as in synthetic control. Within this estimator, the unique decision between collecting new real data vs.\ relying on synthetic controls is reflected by the fact that we allow decompositions of a given subpopulation to include the subpopulation itself.

\justify{\textit{Second,} we provide an upper bound on the variance of our proposed estimator (Proposition~\ref{prop:main}), which can be tighter than the variance of a naive estimator based on conventional controls (Proposition~\ref{prop:corollary}). This upper bound also establishes the rate at which deviations from proper control samples---in favor of synthetic controls---contribute to the error in estimating treatment effects. Identifying this rate leads to an important insight: During an adaptive trial, constructing synthetic controls over collecting real control samples is the most effective when (latent) factors that lead to heterogeneity in patient responses have a stronger effect during the pre-treatment period than the post-treatment period.}

\textit{Finally,} making use of our upper bound, we propose a novel algorithm called \textsc{Syntax} that recruits patients from candidate subpopulations and assigns them into control or treatment groups adaptively to identify subpopulations with positive treatment effect in a sample-efficient manner. Although we motivate and formulate \textsc{Syntax} from the perspective of clinical trials, it remains generally applicable to any thresholding bandit problem with time-series contexts. Through numerical experiments, we validate the performance of our algorithm and provide insight into when it might have an advantage over conventional trial designs.

%%%
\vspace{-3pt}
\section{PROBLEM FORMULATION}
\label{sec:formulation}
\vspace{-3pt}

We are given a pre-specified set of patient subpopulations $[K]=\{1,2,\ldots,K\}$. These subpopulations could have been specified according to any arbitrary criteria (for instance, according to biomarkers or geographical location). Each subpopulation~$i\in[K]$ has known/observable \textit{features} $\bm{x}_i\in\mathbb{R}^{D_x}$ and unknown/unobservable \textit{factor loadings} $\bm{z}_i\in\mathbb{R}^{D_z}$. As an example, the overall cardiovascular health of a subpopulation could be a latent factor that effects heart-related features such as ejection fraction, heart rate, and blood pressure. We let $X=[\bm{x}_1\cdots\bm{x}_K]$ and $Z=[\bm{z}_1\cdots\bm{z}_K]$.
Patients from population~$i$, when they receive no treatment, exhibit the \textit{baseline response} with mean
\begin{align}
    \bar{y}_{it} = \delta_t + \bm{w}_t^{\mathsf{T}}\bm{x}_i + \bm{\mu}_t^{\mathsf{T}}\bm{z}_i \quad\text{for}\quad t\in[T] \label{eqn:synthetic}
\end{align}
\justify{where $\delta_t\in\mathbb{R}$ determines the constant portion of the baseline response, $\bm{w}_t\in\mathbb{R}^{D_x}$ are \textit{weights} that determine the portion based on features, and $\bm{\mu}_t\in\mathbb{R}^{D_z}$ are \textit{factors} that determine the portion based on factor loadings.}

Now, suppose a clinical trial is run over episodes where each episode corresponds to one sample or the recruitment of one patient. At each episode $\eta\in[H]$, the experimenter first recruits a patient from some subpopulation $i[\eta]\in[K]$. Then, they observe the baseline response of the patient they have recruited:
\begin{align}
    y_t[\eta] \sim \mathcal{N}_{\sigma^2}(\bar{y}_{i[\eta]t}) \quad\text{for}\quad t\in\{1,\ldots,T-1\} \label{eqn:_proofa}
\end{align}
where $\mathcal{N}_{\sigma^2}(\mu)$ is the normal distribution with mean $\mu$ and variance $\sigma^2$. This is called the \textit{pre-treatment period}.
Once the pre-treatment period ends at time $t=T$, the experimenter assigns the recruited patient either to the control group, $\alpha[\eta]=0$, or to the treatment group, $\alpha[\eta]=1$. Depending on this assignment, they either continue observing the baseline response or they observe the response affected by the treatment:
\begin{align}
    y_T[\eta] \sim \begin{cases}
        \mathcal{N}_{\sigma^2}(\bar{y}_{i[\eta]T}) &\text{if}\quad \alpha[\eta]=0 \\
        \mathcal{N}_{\sigma^2}(r_{i[\eta]} + \bar{y}_{i[\eta]T}) &\text{if}\quad \alpha[\eta]=1
    \end{cases} \label{eqn:_proofe}
\end{align}
where $r_i$ is the \textit{treatment effect} for subpopulation $i$.

\paragraph{Objective}
The experimenter seeks to identify all of the subpopulations with positive treatment effect
\begin{align}
    \mathcal{I}^* = \{i\in[K]: r_i > 0\}
\end{align}
at the end of the episode horizon $H$, without knowing parameters $\{\delta_t,\bm{w}_t,\bm{\mu}_t\}$ or observing factor loadings $\{\bm{z}_i\}$. However, we assume that features $\bm{x}_i$ are observed ahead of the clinical trial. Denoting with $\hat{\mathcal{I}}^*$ the experimenter's best estimate of $\mathcal{I}^*$, we measure their success through their (i)~\textit{true positive rate} (TPR) given by $\nicefrac{|\hat{\mathcal{I}}^*\cap\mathcal{I}^*|}{|\mathcal{I}^*|}$, and their (ii)~\textit{false positive rate} (FPR) given by $\nicefrac{|\hat{\mathcal{I}}^*\cap[K]\setminus\mathcal{I}^*|}{|[K]\setminus\mathcal{I}^*|}$.

\justify{\paragraph{On Modeling Assumptions}
We formulated the problem of subpopulation selection with a linear model (cf.\ Equation \ref{eqn:synthetic}) and normal error distributions (cf.\ Equation \ref{eqn:_proofa}). It should be emphasized that these are not just arbitrary assumptions that we made because they are common in the literature. Instead, they are careful modeling decisions that are particularly suitable for capturing population-level effects as the units of interest in our problem happen to be populations of patients rather than individual patients themselves. Regarding linearity, \citet{shi2022assumptions} explain how linear relationships can emerge on a population level when outcomes are averaged over multiple individuals, even when the generative process for an individual is non-linear---briefly, this is due to the inherent linearity of the expectation operator. Similarly, error distributions for outcomes averaged over multiple individuals can be approximated as normal due to the central limit theorem. We will evaluate the performance of \textsc{Syntax} under model mismatch---that is when the linearity in Equation~\ref{eqn:synthetic} is violated---during our experiments in Section~\ref{sec:experiments}.}

%%%
\section{ESTIMATING TREATMENT EFFECTS WITH SYNTHETIC CONTROLS}

As it will become more apparent why in Section~\ref{sec:syntax}, efficient exploration of the subpopulations and treatment groups rely on forming estimators of the treatment effects that have low variance. In order to find such estimators, we first start by analyzing the case where recruitment decisions $i[\eta],\alpha[\eta]$ are made independently from observed responses $y_t[\eta]$ so that the observed responses remain independent from each other---that is $y_t[\eta]\perp\!\!\!\perp y_{t'}[\eta']$ if $t\neq t'$ or $\eta\neq\eta'$ for all time steps $t,t'\in[T]$ and episodes $\eta,\eta'\in[H]$.
For a given episode~$\eta$, denote with
\begin{align}
    n_i^{\mathsmaller{(\alpha)}}=\sum\nolimits_{\eta'<\eta}\mathbb{I}\{i[\eta']=i,\allowbreak\alpha[\eta']=\alpha\}
\end{align}
\justify{the number of patients recruited from subpopulation~$i$ into treatment group $\alpha$ until that episode, and with}

\vspace{-\baselineskip-\parskip}
\begin{align}
    \hat{y}_{it}^{\mathsmaller{(\alpha)}}=\sum\nolimits_{\eta'<\eta}y_t[\eta']\cdot\frac{\mathbb{I}\{i[\eta']=i,\allowbreak\alpha[\eta']=\alpha\}}{n_i^{\mathsmaller{(\alpha)}}}
\end{align}
the empirical mean responses of those patients. Similarly, denote with $n_i=n_i^{\mathsmaller{(0)}}+n_i^{\mathsmaller{(1)}}$ the total number of patients recruited from subpopulation $i$ (regardless of their treatment group) and with $\hat{y}_{it}=(\hat{y}_{it}^{\mathsmaller{(0)}}n_i^{\mathsmaller{(0)}}+\hat{y}_{it}^{\mathsmaller{(1)}}n_i^{\mathsmaller{(1)}})/n_i$ the empirical mean responses of those patients. We let $\smash{N^{\mathsmaller{(\alpha)}}}=\smash{\diag(n_1^{\mathsmaller{(\alpha)}},\ldots,n_K^{\mathsmaller{(\alpha)}})}$, $N=\smash{\diag(n_1,\ldots,n_K)}$, $\smash{\bm{\hat{y}}_{\cdot T}^{\mathsmaller{(\alpha)}}}=\smash{[\hat{y}_{1T}^{\mathsmaller{(\alpha)}}\cdots\hat{y}_{KT}^{\mathsmaller{(\alpha)}}]^{\mathsf{T}}}$, $\smash{\bm{\hat{y}}_{i\neg T}}=\smash{[\hat{y}_{i1}\cdots\hat{y}_{i(T-1)}]^{\mathsf{T}}}$, and $\smash{\hat{Y}_{\neg T}}=\smash{[\bm{\hat{y}}_{1\neg T}\cdots\bm{\hat{y}}_{K\neg T}]}$.

Ignoring any potential relationship between the responses of different subpopulations, a naive estimate for the treatment effect $r_i$ can be written as
\begin{align}
    \hat{r}_i^{~\text{naive}} = \hat{y}_{iT}^{\mathsmaller{(1)}} - \hat{y}_{iT}^{\mathsmaller{(0)}} \label{eqn:naive}
\end{align}
which is an unbiased estimate of treatment effects, $\mathbb{E}[r_i-\hat{r}_i^{~\text{naive}}]=0$, and has a variance of $\mathbb{V}[r_i-\hat{r}_i^{~\text{naive}}]=\sigma^2( 1/n_i^{\mathsmaller{(0)}} + 1/n_i^{\mathsmaller{(1)}} )$.
However, considering true mean responses $\bar{y}_{iT}$ are related to each other as described in \eqref{eqn:synthetic}, it should be possible to form other estimates with tighter variances. Specifically, if the joint feature/factor space has lower dimensionality than the number of subpopulations---that is $D_x+D_z < K$---one can explore the feature/factor space more efficiently than the subpopulations themselves. Moreover, if the pre-treatment responses $\{y_t\}_{t<T}$ are observed for a time period long enough to infer unobserved factor loadings $\{\bm{z}_i\}$---that is $T>D_z$---such exploration could be feasible. We aim to achieve this via synthetic control.

\paragraph{Synthetic Control}
Rather than using $\smash{\hat{y}_{iT}^{\mathsmaller{(0)}}}$ as our control when estimating $r_i$, we first decompose subpopulation $i$ as a linear combination of other subpopulations with weights $\bm{\beta}\in\mathbb{R}^K$ such that $\bm{x}_i=X\bm{\beta}$, $\bm{z}_i\approx Z\bm{\beta}$, and the elements of $\bm{\beta}$  sum up to one.%
\footnote{It is intentional here that $\bm{x}_i=X\bm{\beta}$, $\bm{z}_i\approx Z\bm{\beta}$ and \textit{not} $\bm{x}_i=X_{\neg i}\bm{\beta}$, $\bm{z}_i\approx Z_{\neg i}\bm{\beta}$. This is part of our contribution: Unlike previous work on synthetic control, we allow weights~$\bm{\beta}$ to include a subpopulation itself such that $\bm{x}_i=\bm{x}_i\beta_i+X_{\neg i}\bm{\beta}_{\neg i}$, $\bm{z}_i\approx \bm{z}_i\beta_i+Z_{\neg i}\bm{\beta}_{\neg i}$. Of course, one trivial $\bm{\beta}$ that satisfies this condition would be $\beta_i=1$, $\bm{\beta}_{\neg i}=\bm{0}$ (i.e.\ $\bm{\beta}=\bm{1}_i$). However, as we will discuss next, alternative weights can lead to lower-variance estimates.}
Then, \eqref{eqn:synthetic} would imply that $\bar{y}_{iT}\approx \bm{\beta}^{\mathsf{T}}\bm{\bar{y}}_{\cdot T}$ hence we can use $\bm{\beta}^{\mathsf{T}}\bm{\hat{y}}_{\cdot T}^{\mathsmaller{(0)}}$ as our control instead. This leads to a family of \textit{synthetic estimators}:
\begin{align}
    \hat{r}_i(\bm{\beta}) = \hat{y}_{iT}^{\mathsmaller{(1)}} - \bm{\beta}^{\mathsf{T}}\bm{\hat{y}}_{\cdot T}^{\mathsmaller{(0)}} \label{eqn:estimate}
\end{align}
In Proposition~\ref{prop:main}, we show that synthetic estimators become unbiased when $\bm{\beta}$ satisfies verifiable conditions that do not depend on unknown or unobservable quantities such as factors $\bm{\mu}_t$ or factor loadings $\bm{z}_i$. Moreover, we provide an upper bound on their variance. As we will argue next, synthetic estimators include the naive estimate in \eqref{eqn:naive} as well, moreover, the tightness of the variance bound in Proposition~\ref{prop:main} makes it possible to find synthetic estimates with even lower variances than the naive estimate.

\begin{proposition} \label{prop:main}
    \justify{Assuming $M_{\neg T}=[\bm{\mu}_1\cdots\bm{\mu}_{T-1}]$ has full rank and $T>D_z$, we have $\mathbb{E}[r_i-\hat{r}_i(\bm{\beta})] = 0$ and}

    \vspace{-\baselineskip-\parskip}
    ~
    \begin{align}
        &\mathbb{V}[r_i-\hat{r}_i(\bm{\beta})] \nonumber \\
        &\hspace{6pt} \leq V_i(\bm{\beta};N^{\mathsmaller{(0)}},N^{\mathsmaller{(1)}}) \nonumber \\ 
        &\hspace{6pt} \doteq \sigma^2\big( \underbrace{1/n_i^{\mathsmaller{(1)}}+\|\bm{\beta}\|_{(N^{\mathsmaller{(0)}})^{-1}}^2}_{\text{\normalfont epistemic uncertainty}} + \underbrace{\lambda\|\bm{\beta}-\bm{1}_i\|_{N^{-1}}^2}_{\text{\normalfont representation error}} \!\!\big)
    \end{align}
    for $\lambda = \|M_{\neg T}^{\mathsf{T}}(M_{\neg T}M_{\neg T}^{\mathsf{T}})^{-1}\bm{\mu}_T\|^2$ when $\bm{\beta}$ is such that $\bm{x}_i=X\bm{\beta}$, $\bm{\hat{y}}_{i\neg T}=\hat{Y}_{\neg T}\bm{\beta}$, and $\bm{1}^{\mathsf{T}}\bm{\beta}=1$.
\end{proposition}

\begin{proof}
    All proofs can be found in the appendix.
\end{proof}

\subsection{Interpreting Proposition~\ref{prop:main}}

To gain a more intuitive understanding of the bound in Proposition~\ref{prop:main}, first consider the trivial decomposition of subpopulation $i$ as itself---that is $\bm{\beta}=\bm{1}_i$ where $(\bm{1}_i)_j=\mathbb{I}\{i=j\}$. Normally, having $\smash{\bm{\hat{y}}_{i\neg T}=\hat{Y}_{\neg T}\bm{\beta}}$ would not guarantee $\bm{z}_i=Z\bm{\beta}$ due to the observation noise. However, for the trivial decomposition uniquely, we know for certain that $\bm{z}_i=Z\bm{\beta}$ holds. In other words, $\bm{\beta}=\bm{1}_i$ is the only known representation of subpopulation $i$ that is \textit{perfectly} accurate even in the latent factor space. The \textit{representation error} $\lambda\|\bm{\beta}-\bm{1}_i\|^2_{N^{-1}}$ takes deviations from this perfect representation into account when quantifying the variance of synthetic estimators. Synthetic control relies on matching observed responses $\hat{y}$ between a subpopulation~$i$ and its representation~$\bm{\beta}$ such that $\bm{\hat{y}}_{i\neg T}=\hat{Y}_{\neg T}\bm{\beta}$ as it leads to a near match in terms of unobservable factor loadings $\bm{z}_i\approx Z\bm{\beta}$ as well, and notably, the deviations from the perfect representation are weighted by $N^{-1}$ which corresponds to the uncertainty of observed responses $\smash{\hat{Y}_{\neg T}}$.

The remaining terms capture the \textit{epistemic uncertainty} of $\hat{y}_{iT}^{\mathsmaller{(1)}}$ and $\bm{\hat{y}}_{\cdot T}^{\mathsmaller{(0)}}$ which are composed together to form the final estimate $\hat{r}_i(\bm{\beta})$. Even when $\bm{\beta}$ is a perfect representation, the accuracy of $\hat{r}_i(\bm{\beta})$ is limited by how close $\hat{y}_{iT}^{\mathsmaller{(1)}}$ and $\bm{\hat{y}}_{\cdot T}^{\mathsmaller{(0)}}$ are to their ground-truth values.

\justify{Importantly, the trivial representation $\bm{\beta}=\bm{1}_i$ (of subpopulation $i$ as itself) recovers the naive estimate in \eqref{eqn:naive}:}

\vspace{-\baselineskip-\parskip}
~
\begin{align}
    \hat{r}_i(\bm{1}_i)=\hat{r}_i^{~\text{naive}}
\end{align}
Moreover, when $\bm{\beta}=\bm{1}_i$, the representation error in $V_i(\bm{\beta})$ vanishes and the epistemic uncertainty becomes identical to
% the naive estimate's variance in \eqref{eqn:naive-variance}:
% \begin{align}
%     V_i(\bm{1}_i) = \mathbb{V}[r_i-\hat{r}_i^{~\text{naive}}] \label{eqn:_proofi}
% \end{align}
the naive estimate's variance---that is $V_i(\bm{1}_i) = \mathbb{V}[r_i-\hat{r}_i^{~\text{naive}}]$.
This means that the family of estimators $\hat{r}_i(\bm{\beta})$ not only include the naive estimate but also the variance bound in Proposition~\ref{prop:main} is optimally tight for the naive estimate.

Having made this important observation, the key idea behind \textsc{Syntax} is to search for $\bm{\beta}$ that leads to even tighter variance bounds $V_i(\bm{\beta})$. This can easily be achieved by solving the quadratic program
\begin{align}
    % \bm{\beta}^*_i = {\argmin_{\bm{\beta}:}}_{\subalign{\bm{x}_i&=X\bm{\beta}\\\bm{\hat{y}}_{i\neg T}&=\hat{Y}_{\neg T}\bm{\beta}\\\bm{1}^{\mathsf{T}}\bm{\beta}&=1}} V_i(\bm{\beta}) \label{eqn:_proofj}
    \bm{\beta}^*_i = \argmin_{\bm{\beta}:\,\bm{x}_i=X\bm{\beta},\,\bm{\hat{y}}_{i\neg T}=\hat{Y}_{\neg T}\bm{\beta},\,\bm{1}^{\mathsf{T}}\bm{\beta}=1} V_i(\bm{\beta}) \label{eqn:_proofj}
\end{align}
Due to Proposition~\ref{prop:main}, it is guaranteed that $\hat{r}_i(\bm{\beta}^*_i)$ has variance at least as small as that of the naive estimator.

\begin{proposition} \label{prop:corollary}
    $\mathbb{V}[r_i-\hat{r}_i(\bm{\beta}^*_i)]\leq\mathbb{V}[r_i-\hat{r}_i^{~\text{naive}}]$.
\end{proposition}

\subsection{When is synthetic control the most effective in estimating treatment effects?}
\label{sec:intuition}

Proposition~\ref{prop:main} provides an intuition as to when synthetic control might be the most effective in estimating treatment effects. Notice that, in Proposition~\ref{prop:main}, the representation error contributes to the total variance proportionally to constant $\lambda = \|M_{\neg T}^{\mathsf{T}}(M_{\neg T}M_{\neg T}^{\mathsf{T}})^{-1}\bm{\mu}_T\|^2$. Earlier, we have discussed how representation error is caused mainly due to a mismatch between a subpopulation's factor loadings $\bm{z}_i$ and the synthetic composition $Z^{\mathsf{T}}\bm{\beta}$. As such, it is no surprise that $\lambda$ exclusively depends on factors~$\bm{\mu}_t$. Since
\begin{align}
        % \lambda = \|M_{\neg T}^{\mathsf{T}}(M_{\neg T}M_{\neg T}^{\mathsf{T}})^{-1}\bm{\mu}_T\|^2 &\leq \bigg(\frac{\|\bm{\mu}_T\|}{\|M_{\neg T}\|}\bigg)^2
        \lambda = \|M_{\neg T}^{\mathsf{T}}(M_{\neg T}M_{\neg T}^{\mathsf{T}})^{-1}\bm{\mu}_T\|^2 &\leq (\nicefrac{\|\bm{\mu}_T\|}{\|M_{\neg T}\|})^2
\end{align}
we also expect that $\lambda$---hence the contribution of representation error to the total variance---is smaller when post-treatment factors are small in magnitude (i.e.\ $\|\bm{\mu}_T\|$ is small) and pre-treatment factors are large in magnitude (i.e.\ $\|M_{\neg T}\|$).

This makes intuitive sense as well: With synthetic control, we are essentially trying to infer factor loadings $\bm{z}_i$ from pre-treatment responses $\{y_t\}_{t<T}$ in order to estimate the post-treatment responses $y_T$. During the pre-treatment period, if the contribution of factors to responses is stronger relative to the noise level---that is when $\|M_{\neg T}\|/\sigma$ is large---we expect our inference of factor loadings to be more accurate. Later, when we estimate the post-treatment response $y_T$, if the contribution of factors to the final response is now weaker relative to the noise level---that is when $\|\bm{\mu}_T\|/\sigma$ is small---we expect to gain more from our earlier \mbox{inference} of the factor loadings compared with a naive estimate (which now has to detect a weaker signal without any additional information from the previous time steps $t<T$ during the pre-treatment period). We will confirm this intuition empirically with simulations in Section~\ref{sec:experiments}.

%%%
\section{SYNTAX}
\label{sec:syntax}

Making use of the synthetic estimator~$\hat{r}_i(\bm{\beta}^*)$ given by \eqref{eqn:_proofj} and the upper bound on its variance given in Proposition~\ref{prop:main}, we propose \textsc{Syntax} described in Algorithm~\ref{alg:syntax}. It is an adaptation of the algorithm of \citet{locatelli2016optimal} for solving thresholding bandit problems. At each episode~$\eta\in[H]$:

\begin{enumerate}[label=(\roman*),nosep,parsep=\parskip,leftmargin=18pt]
    \item First, a \textit{sensitivity index} is computed for each subpopulation $i\in[K]$:
    \begin{align}
        S_i = \frac{|\hat{r}_i(\bm{\beta}^*_i)|}{\sqrt{V_i(\bm{\beta}^*_i)}} \leq \frac{|\hat{r}_i(\bm{\beta}^*_i)|}{\sqrt{\mathbb{V}[r_i-\hat{r}_i(\bm{\beta}^*_i)]}}
    \end{align}
    Intuitively, the lower the sensitivity index of a subpopulation is, the harder it is to determine whether its treatment effect is positive or not.
    \item Then, the subpopulation~$i^*=\argmin S_i$ with the lowest sensitivity index is determined and the patient that is expected to cause the largest increase in $S_{i^*}$ is recruited.
\end{enumerate}

\begin{algorithm}[H]
    \caption{\textsc{Syntax}}
    \label{alg:syntax}
    \vspace{2.7pt}
    \textbf{Parameters:} Horizon~$H$, factor effect parameter~$\lambda$ \\
    \textbf{Output:} Subpopulations~$\hat{\mathcal{I}}^*$ with positive treat.\ effect%
    \vspace{2.7pt}
    \hrule
    \begin{algorithmic}[1]
        \STATE $n_i^{\mathsmaller{(\alpha)}}\gets 0$, $\hat{y}_{iT}^{\mathsmaller{(\alpha)}}\gets 0$, $\bm{\hat{y}}_{i\neg T}\gets\bm{0}$, $\forall i\in[K],\alpha\in\{0,1\}$%
        \FOR{$\eta\in\{1,2,\ldots,H\}$}
            \STATE $i^* \gets \argmin_{i\in[K]} S_i$
            \STATE $i,\alpha \gets \argmin_{i\in[K],\alpha\in\{0,1\}}$ \\
            \hspace{26.5pt} $\min_{\bm{\beta}:\: \bm{x}_{i^*}=X\bm{\beta},\, \bm{\hat{y}}_{i^*\neg T}=\hat{Y}_{\neg T}\bm{\beta},\,\bm{1}^{\mathsf{T}}\bm{\beta}=1}$ \\
            \hspace{26.5pt} $V_{i^*}(\bm{\beta}; N^{\mathsmaller{(0)}}\!+\!(1-\alpha)\bm{1}_i{\bm{1}_i}^{\mathsf{T}}, N^{\mathsmaller{(1)}}\!+\!\alpha\bm{1}_i{\bm{1}_i}^{\mathsf{T}})\!$
            \STATE Recruit from population~$i$ and treatment group~$\alpha$%
            \STATE Observe baseline outcomes $\bm{y}_{\neg T}=[y_1\cdots y_{T-1}]^{\mathsf{T}}$ 
            \\ \hfill and the final outcome $y_T$
            \STATE $n_i^{\mathsmaller{(\alpha)}}\gets n_i^{\mathsmaller{(\alpha)}}+1$
            \STATE $\bm{\hat{y}}_{i\neg T}\gets \bm{\hat{y}}_{i\neg T} + (\bm{y}_{\neg T}-\bm{\hat{y}}_{i\neg T})/n_i$
            \STATE $\hat{y}_{iT}^{\mathsmaller{(\alpha)}}\gets \hat{y}_{iT}^{\mathsmaller{(\alpha)}} + (y_T-\hat{y}_{iT}^{\mathsmaller{(\alpha)}})/n_i^{\mathsmaller{(\alpha)}}$
        \ENDFOR
        \STATE $\hat{\mathcal{I}}^*\gets \{i\in[K]: \hat{r}_i(\bm{\beta}^*_i) > 0\}$
    \end{algorithmic}
\end{algorithm}

Finally, when the experiment ends, \textsc{Syntax} simply returns all subpopulations that have a positive treatment effect estimate: $\hat{\mathcal{I}}^*=\{i\in[K]:\hat{r}_i(\bm{\beta}^*_i)\}$. Note that, computing $V_i(\bm{\beta})$ hence sensitivity indices $S_i$ requires specifying $\lambda$. The ideal $\lambda$ given Proposition~\ref{prop:main} cannot be computed in practice as it depends on unknown factors $\bm{\mu}_t$, and we call it the \textit{factor effect parameter} due to this dependence.

%%%
\section{RELATED WORK}

\begin{table*}[t]
    \centering
    \caption{\textbf{Comparison of related trial designs.} \textsc{Syntax} makes inferences via synthetic control and recruits adaptively.}%
    \label{tab:related}%
    \vspace{-\baselineskip+6pt}

    \small
    \resizebox{.9\linewidth}{!}{%
        \begin{tabular}{@{}llll@{}}
            \toprule
            \bf Approach & \bf Inference Strategy & \bf Sampling Strategy & \bf Related Work \\
            \midrule
            Conventional studies & \multirow{2}{*}{\makecell[l]{Simple statistics\\\scriptsize$\hat{r}_i=\hat{y}_{iT}^{\mathsmaller{(1)}}-\hat{y}_{iT}^{\mathsmaller{(0)}}$}} & Randomized & -- \\
            Thresholding bandits & & Adaptive & \citet{locatelli2016optimal} \\
            \midrule
            Synthetic studies & \multirow{3}{*}{\makecell[l]{Synthetic control\\\scriptsize$\hat{r}_i=\hat{y}_{iT}^{\mathsmaller{(1)}}-(\bm{\beta}^*_i=\argmin_{\bm{\beta}} V_i(\bm{\beta}))^{\mathsf{T}}\bm{\hat{y}}_{\cdot T}^{\mathsmaller{(0)}}$}} & Randomized & \citet{abadie2015comparative} \\
            Synthetic design & & Pre-planned & \citet{doudchenko2021synthetic} \\
            \textbf{\textsc{Syntax}} & & Adaptive & \textbf{(Ours)} \\
            \bottomrule
        \end{tabular}}%
        \vspace{-\baselineskip+3pt}
\end{table*}

Our objective in this paper is to identify subpopulations with positive treatment effect given a fixed budget of samples. A trial designed to do so would consist of (i) an \textit{inference strategy} that dictates which subpopulation is identified at the end of the trial, and (ii) a \textit{sampling strategy} that dictates how the samples are allocated between different subpopulations as well as control and treatment groups. Our trial design, \textsc{Syntax}, happen to combine inference techniques from the \textit{synthetic control} literature and sampling principals originally developed for solving \textit{thresholding bandit} problems. Table~\ref{tab:related} summarizes alternative trial designs one might consider, which differ from \textsc{Syntax} in terms of their inference and sampling strategies. We give an overview of these alternative strategies in this section.

\paragraph{Synthetic Control}
\justify{It is not always practical to perform large randomized experiments to understand the effects of an intervention \citep{bica2020time}. For instance, investigating the effects of new policies targeting large geographic areas, often at the level of individual countries, is challenging. As an inference technique, synthetic control \citep{abadie2003abadie,abadie2010synthetic,abadie2015comparative} was first introduced to address this challenge. In typical synthetic control studies, the outcome observed for a single treated unit is compared against a control unit that is synthetically generated as a linear combination of other untreated units.}

\justify{More recently, \citet{doudchenko2021synthetic} introduced a trial design based on synthetic control, where treated and untreated units are not fixed but rather chosen by the designer before the experiment is run. Notably though, their trial design does not adapt to the outcomes observed sequentially once the experiment starts. In contrast, \textsc{Syntax} recruits and assigns patients---which is akin to choosing which units will be treated---in a fully adaptive way. \citet{farias2022synthetically} study the adaptive use of synthetic controls but in a setting orthogonal to ours: They consider treatment assignments over time steps $t$ for a single unit/patient whereas we consider the treatment assignment at a final time point $T$ for multiple units/patients over multiple episodes $\eta$.}

\paragraph{Thresholding Bandits}
As an online algorithm, \textsc{Syntax} is closely related to the multi-armed bandit literature \citep[e.g.][]{auer2002finite,huyuk2019analysis,huyuk2020thompson,huyuk2021multi}, and
our objective is most similar to that of the thresholding bandit problem \citep{locatelli2016optimal,zhong2017asynchronous,mukherjee2017thresholding,tao2019chao} and the good arm identification problem \citep{kano2019good,katz2020true}, which similarly aim to identify arms with rewards higher than a given threshold among a set of candidate arms. Subpopulations in our work can be thought of as arms in a thresholding bandit problem with one important caveat: While in thresholding bandits, each arm has only one reward distribution associated with it, in our setting, the treatment effect of each subpopulation depends on two response distributions that cannot be sampled from simultaneously. Complicating matters further, while there are underlying relationships between the responses of different subpopulations for the control group as described in \eqref{eqn:synthetic}, no such relationship exists for the treatment group.

Another related problem in the multi-armed bandit literature is the best arm identification (BAI) problem \citep{bubeck2009pure,audibert2010best,gabillon2011multi,gabillon2012best,soare2014best,xu2018fully,alieva2021robust,degenne2020gamification}. In our context, BAI would correspond to finding the subpopulation with the largest treatment effect. This objective is clinically less suitable as it is important to provide the treatment to all patients who would benefit from it, regardless of how little that benefit might be compared to other patients. Aside from ``pure-exploration'' problems such as thresholding bandits or BAI, there are numerous other work on multi-armed bandits that focus on maximizing the cumulative reward instead. In our context, this would correspond to maximizing the benefit received by all participant of a trial whereas clinical trials are usually exploratory in nature.

\paragraph{Adaptive Clinical Trials}
Finally, it should be mentioned that adaptive trial designs that target objectives other than ours exist. \citet{lewis1990sequential,whitehead1997design} consider when to terminate a trial, \citet{hu2006theory,berry2006bayesian,villar2015multi} maximizes the benefit for the recruited patients, \citet{bhatt2006adaptive} consider when to stop further investigating one of two subpopulations, \citet{oquigley1990continual,riviere2014bayesian,wages2015phase,yan2017keyboard,shen2020learning,lee2020contextual,lee2021sdfbayes} determine safe but effective dosage for treatments, and \citet{huyuk2023when} consider the portfolio-level management of trials. Notably, \citet{onur2019sequential,curth2023adaptively} also aim to identify subpopulations with positive treatment effects but not in the setting where pre-treatment responses are available for all participants.
Adaptive enrichment designs \citep{stallard2014adaptive} perform interim analyses to select subpopulations. However there, the focus is to seamlessly adapt the inclusion criteria of a confirmatory trial without compromising its type I error rate. In contrast, we focus purely on selecting subpopulations in an exploratory trial that can then be targeted in a confirmatory trial.

%%%
\vspace{-3pt}
\section{EXPERIMENTS}
\label{sec:experiments}
\vspace{-3pt}

\begin{table*}
    \caption{\textbf{Performance comparison of \textsc{Syntax} and benchmarking algorithms.} Inline with our intuition from Section~\ref{sec:intuition}, \textsc{Syntax} performs the best for \textit{Diminishing Factor Effects}. Again inline with our intuition, it does not provide any benefit for \textit{Increasing Factor Effects} but still performs on par with \textit{Thresholding Bandits} since synthetic estimators have variances at least as low as the naive estimate (as stated in Proposition~\ref{prop:corollary}).}%
    \label{tab:results}%
    \vspace{-3pt}%
    
    \small
    \resizebox{\linewidth}{!}{%
        \begin{tabular}{@{}l*{8}{c@{~}c}@{}}
            \toprule
            & \multicolumn{8}{c}{\textbf{Diminishing Factor Effects}} & \multicolumn{8}{c}{\textbf{Increasing Factor Effects}} \\
            \cmidrule(lr){2-9} \cmidrule(l){10-17} 
            & \multicolumn{4}{c}{$H=200$} & \multicolumn{4}{c}{$H=400$} & \multicolumn{4}{c}{$H=200$} & \multicolumn{4}{c}{$H=400$} \\
            \cmidrule(lr){2-5} \cmidrule(lr){6-9} \cmidrule(lr){10-13} \cmidrule(l){14-17}
            \textbf{Algorithm} & \multicolumn{2}{c}{FPR} & \multicolumn{2}{c}{TPR} & \multicolumn{2}{c}{FPR} & \multicolumn{2}{c}{TPR} & \multicolumn{2}{c}{FPR} & \multicolumn{2}{c}{TPR} & \multicolumn{2}{c}{FPR} & \multicolumn{2}{c}{TPR} \\
            \midrule
            Conventional study
                & 19.5\% & (0.2\%) & 80.7\% & (0.3\%)
                & 14.9\% & (0.3\%) & 85.4\% & (0.3\%)
                & 19.5\% & (0.2\%) & 80.7\% & (0.3\%)
                & 14.9\% & (0.3\%) & 85.4\% & (0.3\%)
                \\
            Thresholding bandits
                & 17.6\% & (0.4\%) & 82.6\% & (0.4\%)
                & 13.7\% & (0.4\%) & 86.4\% & (0.2\%)
                & \bf 17.6\% & \bf (0.4\%) & \bf 82.6\% & \bf (0.4\%)
                & \bf 13.7\% & \bf (0.4\%) & \bf 86.4\% & \bf (0.2\%)
                \\
            Synthetic study
                & 16.7\% & (0.3\%) & 83.4\% & (0.3\%)
                & 12.5\% & (0.3\%) & 87.7\% & (0.2\%)
                & 19.5\% & (0.2\%) & 80.7\% & (0.3\%)
                & 14.9\% & (0.3\%) & 85.4\% & (0.3\%)
                \\
            Synthetic design
                & 16.4\% & (0.4\%) & 83.8\% & (0.4\%)
                & 12.1\% & (0.4\%) & 88.2\% & (0.3\%)
                & 19.7\% & (0.4\%) & 80.5\% & (0.3\%)
                & 14.9\% & (0.3\%) & 85.4\% & (0.4\%)
                \\
            \midrule
            \textbf{\textsc{Syntax}}
                & \bf 14.6\% & \bf (0.4\%) & \bf 85.6\% & \bf (0.3\%)
                & \bf 11.0\% & \bf (0.3\%) & \bf 89.1\% & \bf (0.2\%)
                & \bf 17.5\% & \bf (0.4\%) & \bf 82.6\% & \bf (0.3\%)
                & \bf 13.7\% & \bf (0.4\%) & \bf 86.4\% & \bf (0.3\%)
                \\
            \bottomrule
        \end{tabular}}%
        \vspace{-3pt}%
\end{table*}

\begin{figure*}[t]
    \centering
    \begin{subfigure}{.5\linewidth}
        \includegraphics[width=.5\linewidth]{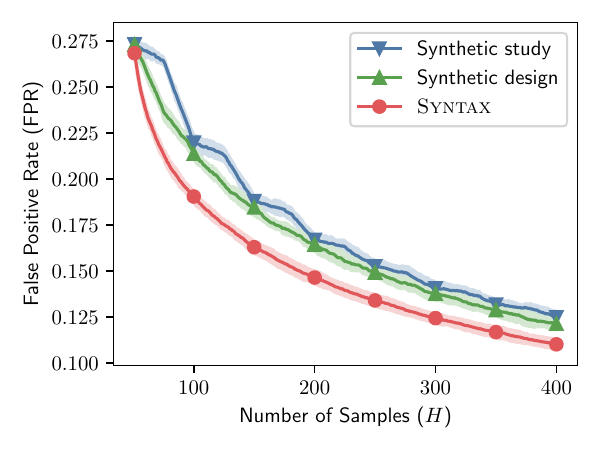}%
        \includegraphics[width=.5\linewidth]{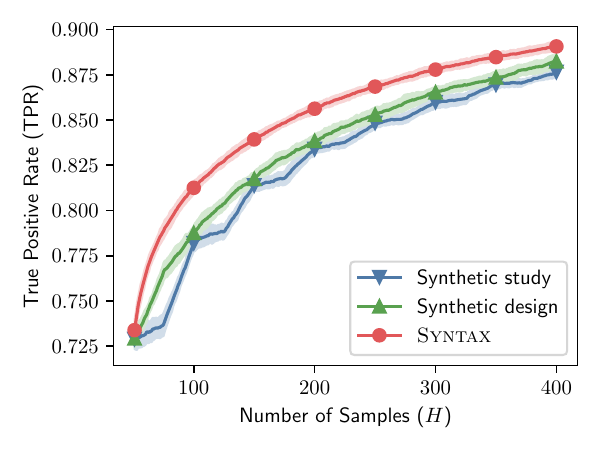}%
        \vspace{-6pt}%
        \caption{\textit{Synthetic Algorithms}}%
        \label{fig:exp-synthetic}%
    \end{subfigure}%
    \begin{subfigure}{.5\linewidth}
        \includegraphics[width=.5\linewidth]{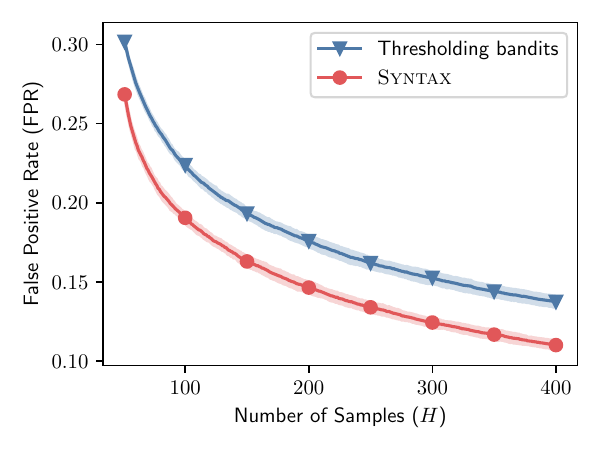}%
        \includegraphics[width=.5\linewidth]{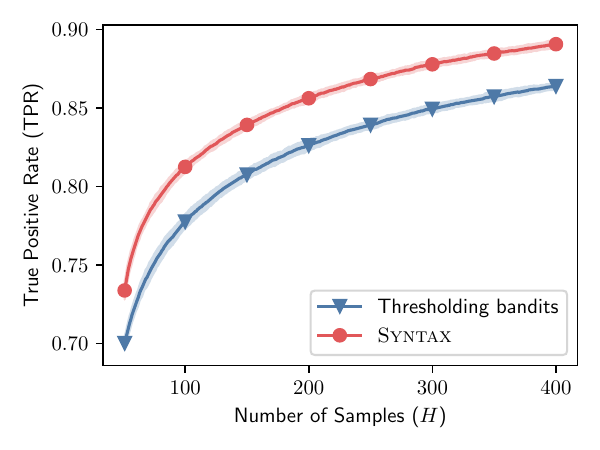}%
        \vspace{-6pt}%
        \caption{\textit{Adaptive Algorithms}}%
        \label{fig:exp-adaptive}%
    \end{subfigure}%
    \vspace{-\baselineskip+6pt}%
    \caption{\textbf{Comparison of (a) synthetic algorithms and (b) adaptive algorithms.} Switching to a pre-planned sampling strategy from an adaptive one or switching to a naive inference strategy from a synthetic one both cause FPR to increase and TPR to decrease at comparable scales.}%
    \vspace{-\baselineskip}%
\end{figure*}

\paragraph{Environments}
To confirm our earlier intuition in Section~\ref{sec:intuition} about when synthetic control might be the most effective, we consider two types of environments:

\begin{enumerate}[label=(\roman*),nosep,parsep=\parskip,leftmargin=18pt]
    \item \textit{Diminishing Factor Effects:} In these environments, the contributions of factors $\bm{\mu}_t$ to the baseline responses gets weaker over time. Formally, we set $\bm{\mu}_t = (2 - 10^{t-T})\bm{\mu'}_t$, where $\bm{\mu'}_t$ is sampled uniformly at random from the unit ball in $\mathbb{R}^{D_z}$.
    
    \item \textit{Increasing Factor Effects:} In contrast to \textit{Diminishing Factor Effects}, in these environments, the contributions of factors $\bm{\mu}_t$ to the baseline responses gets stronger over time. Formally, we set $\bm{\mu}_t = 10^{t-T}\bm{\mu'}_t$, where $\bm{\mu'}_t$ is sampled uniformly at random from the unit ball in $\mathbb{R}^{D_z}$.
\end{enumerate}

Remember that we expect our gain from synthetic control to be larger for \textit{Diminishing Factor Effects} since we infer factor loadings based on earlier responses---for which we would prefer factors to  provide a stronger signal---and we aim to estimate later responses---for which we would gain the most from our earlier inference if factors now provide a weaker signal that is harder to pick up with naive estimates.
We set $N=25$, $T=5$, $D_x=D_z=2$. Notably, $D_x+D_z<N$ and $T>D_z$. We sample the remaining parameters of the environment so that $x_{ij}\sim\mathcal{N}_1(0)$, $z_{ij}\sim\mathcal{N}_1(0)$, $\delta_t\sim\mathcal{N}_1(0)$, $r_i\sim\mathcal{N}_1(0)$, and $\bm{w}_t$ is picked uniformly at random from the unit ball in $\mathbb{R}^{D_x}$. We repeat all our experiments ten times to obtain error bars. During each repetition, we sample a new environment and we average FPR and TPR of all our benchmark algorithms over 1000 runs.

\justify{\paragraph{Benchmarks}
We adapt the alternative trial designs summarized in Table~\ref{tab:related} to our problem setting as benchmarks. Adapting \textit{Thresholding Bandits}, when there are two response distributions associated with each subpopulation rather than a single reward distribution, involves forming naive estimates as in \eqref{eqn:naive}. \textit{Synthetic Design} normally requires knowing the mean responses for the pre-treatment period before forming a fixed sampling plan. When adapting it, we allowed the sampling plan to be updated using the latest empirical mean responses $\hat{y}_{i\neg T}^{\mathsmaller{(\alpha)}}$ at each episode. However, unlike \textsc{Syntax}, \textit{Synthetic Design} still does not take the final responses $y_T$ into consideration when allocating samples.  Details regarding each benchmark can be found in the appendix.}

\paragraph{Main Results}
Table~\ref{tab:results} compares \textsc{Syntax} and other benchmarks in terms of their performance. We consider two metrics: (i) \textit{false positive rate} (FPR), which is the proportion of subpopulations incorrectly identified as having positive treatment effect among all subpopulations with negative treatment effect, and (ii) \textit{true positive rate} (TPR), which is the proportion of subpopulations correctly identifies as having positive treatment effect among all subpopulations with positive treatment effect. We see that our intuition from Section~\ref{sec:intuition} holds: \textsc{Syntax} performs the best for \textit{Diminishing Factor Effects}. Although it does not provide any additional benefit for \textit{Increasing Factor Effects}, it still performs on par with \textit{Thresholding Bandits}, which relies on the naive estimate in \eqref{eqn:naive}. This is because synthetic estimators given by \eqref{eqn:_proofj} have variances at least as low as the naive estimate (i.e.\ they are as informative) as stated in Proposition~\ref{prop:corollary}.

%In order to get a better sense of 
\justify{To understand
how much impact each design aspect of \textsc{Syntax} has on its performance, we compare all three synthetic algorithms in Figure~\ref{fig:exp-synthetic} and the two adaptive algorithms in Figure~\ref{fig:exp-adaptive} (for \textit{Diminishing Factor Effects}). These figures show how the FPR/TPR of different algorithms improve as the sample budget
%(or episode horizon) 
$H$ gets larger. We see that switching to a naive inference strategy or a pre-planned sampling strategy both result in similar performance drops---indicating that synthetic and adaptive aspects of \textsc{Syntax} are equally important.}

\paragraph{Significance of the Main Results}
Looking at Table~\ref{tab:results}, the performance difference between \textsc{Syntax} and a conventional RCT might not seem like much at a first glance. After all, it is merely (!) a \newtilde 5\% difference in terms of FPR/TPR for \textit{Diminishing Factor Effects} with $H=200$. However, the implications of this seemingly small difference are significant in the larger context of clinical trials. We highlight two points: (i)~how hard it actually is to gain this \newtilde 5\% difference and (ii) what tangible benefits it would offer if the same gain were to be realized in practice:
\begin{enumerate}[label=(\roman*),nosep,parsep=\parskip,leftmargin=18pt]
    \item \textcolor{black}{Notice that the performance of a traditional RCT in terms of FPR/TPR also improves only by \newtilde 5\% when the sample size is increased to $H=400$. This means that achieving the same performance gain as \textsc{Syntax} with a traditional RCT requires doubling the size of the RCT---meaning just by allocating control samples smartly, \textsc{Syntax} is able to match the performance of a trial that has double the size.}
    
    \item \textcolor{black}{A typical RCT costs 12-35 million USD, which primarily scales with the number of patients involved \citep{moore2018estimated}, and takes 1-2 years, where as much as 86\% of trials get delayed due to failures to reach recruitment targets \citep{huang2018clinical}. Having the required number of samples effectively halved by designs like \textsc{Syntax} would save upwards of 17.5 million USD per trial---note that more than 9000 trials are launched each year in the US alone \citep{worldnumber}---and help deliver much faster results.}
\end{enumerate}

Similarly, the 1-2\% difference between \textsc{Syntax} and the next best-performing benchmark, \textit{Synthetic Design}, for \textit{Diminishing Factor Effects}, translates into a significant difference in terms of the sample size: \textsc{Syntax} with $H=150$ achieves an FPR of 16.3\% (0.4\%) and TPR of 83.9\% (0.2\%)---the same performance as synthetic design with $H=200$ but using 25\% fewer samples.

\begin{figure}
    \centering
    \includegraphics[width=.8\linewidth]{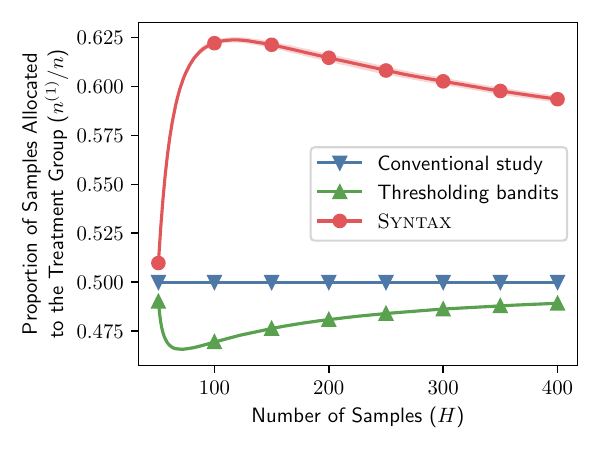}%
    \vspace{-\baselineskip}%
    \caption{\textbf{Proportion of samples allocated to the treatment group over the control group.} By sharing information between the control samples of different subpopulations, \textsc{Syntax} is able allocate more of its samples to the treatment group compared with alternative designs.}%
    \label{fig:samples}%
    \vspace{-\baselineskip-6pt}%
\end{figure}

\paragraph{Sample Allocation Characteristics}
% In the introduction, 
We mentioned how sharing information between the control samples of different subpopulations would allow more samples to be allocated to the treatment group. Figure~\ref{fig:samples} shows the proportion of samples allocated to the treatment group by different algorithms (for \textit{Diminishing Factor Effects}). Indeed, \textsc{Syntax} allocates more of its samples to the treatment group. Proposition~\ref{prop:main} might provide insight into how this is achieved: The uncertainty of the treated response of each subpopulation~$i$ contributes directly to the variance of synthetic estimators through the term $1/n_i^{\mathsmaller{(1)}}$. Whereas, it is possible to mitigate the uncertainty of baseline responses by distributing its contribution among multiple different subpopulations---more precisely, $\min_{\bm{\beta}}\|\bm{\beta}\|_{(N^{\mathsmaller{(0)}})^{-1}}\leq\|\bm{1}_i\|_{(N^{\mathsmaller{(0)}})^{-1}}=1/n_i^{\mathsmaller{(0)}}$. In contrast to \textsc{Syntax}, a conventional study allocates all samples uniformly at random and \textit{Thresholding Bandits} distributes samples adaptively among subpopulations but tends to not differentiate between different treatment groups.

\paragraph{Performance under Model Mismatch}
We also evaluate the performance of \textsc{Syntax} under model mismatch---specifically when the linearity assumption in \eqref{eqn:synthetic} is violated (although it is a natural assumption to make in our setting, see ``On Modeling Assumptions'' in Section~\ref{sec:formulation}). This can only happen in two ways: (i)~Outcomes $y_{it}$ might not be linear with respect to features $\bm{x}_i$. (ii)~Outcomes $y_{it}$ can always be expressed as a linear function of \textit{some} latent factors $\bm{z}_i$, however the minimum number of factors needed to do so might be equal to the number of total time steps (i.e.\ $D_{z}=T$).

\begin{table}
    \centering
    \caption{\textbf{Performance comparison of \textsc{Syntax} and benchmarking algorithms under model mismatch}---when outcomes are non-linear with respect to features.}%
    \label{tab:results-additional-a}%
    \vspace{-7.2pt}%

    \small
    \resizebox{\linewidth}{!}{%
        \begin{tabular}{@{}l*{4}{c@{~}c}@{}}
            \toprule
            & \multicolumn{4}{c}{$H=200$} & \multicolumn{4}{c}{$H=400$} \\
            \cmidrule(lr){2-5} \cmidrule(lr){6-9}
            \textbf{Algorithm} & \multicolumn{2}{c}{FPR} & \multicolumn{2}{c}{TPR} & \multicolumn{2}{c}{FPR} & \multicolumn{2}{c}{TPR} \\
            \midrule
            Conventional study
                & 19.5\% & (0.2\%) & 80.7\% & (0.3\%)
                & 14.9\% & (0.3\%) & 85.4\% & (0.3\%)
                \\
            Thresholding bandits
                & 17.6\% & (0.4\%) & 82.6\% & (0.4\%)
                & 13.7\% & (0.4\%) & 86.4\% & (0.2\%)
                \\
            Synthetic study
                & 18.1\% & (0.4\%) & 81.8\% & (0.3\%)
                & 14.2\% & (0.4\%) & 85.9\% & (0.3\%)
                \\
            Synthetic design
                & 18.3\% & (0.5\%) & 82.0\% & (0.4\%)
                & 14.2\% & (0.4\%) & 85.8\% & (0.3\%)
                \\
            \midrule
            \textbf{\textsc{Syntax}}
                & \bf 16.2\% & \bf (0.4\%) & \bf 84.0\% & \bf (0.4\%)
                & \bf 12.9\% & \bf (0.4\%) & \bf 87.3\% & \bf (0.3\%)
                \\
            \bottomrule
        \end{tabular}}%
    \vspace{-3pt}
\end{table}

\begin{table}
    \centering
    \caption{\textbf{Performance comparison of \textsc{Syntax} and benchmarking algorithms under model mismatch}---when the number of latent factors $D_z=T$.}%
    \label{tab:results-additional-b}%
    \vspace{-7.2pt}%

    \small
    \resizebox{\linewidth}{!}{%
        \begin{tabular}{@{}l*{4}{c@{~}c}@{}}
            \toprule
            & \multicolumn{4}{c}{$H=200$} & \multicolumn{4}{c}{$H=400$} \\
            \cmidrule(lr){2-5} \cmidrule(lr){6-9}
            \textbf{Algorithm} & \multicolumn{2}{c}{FPR} & \multicolumn{2}{c}{TPR} & \multicolumn{2}{c}{FPR} & \multicolumn{2}{c}{TPR} \\
            \midrule
            Conventional study
                & 19.5\% & (0.2\%) & 80.7\% & (0.3\%)
                & 14.9\% & (0.3\%) & 85.4\% & (0.3\%)
                \\
            Thresholding bandits
                & 17.6\% & (0.4\%) & \bf 82.6\% & \bf (0.4\%)
                & 13.7\% & (0.4\%) & \bf 86.4\% & \bf (0.2\%)
                \\
            Synthetic study
                & 15.8\% & (0.3\%) & 71.2\% & (0.7\%)
                & 12.3\% & (0.3\%) & 74.9\% & (0.9\%)
                \\
            Synthetic design
                & 15.7\% & (0.3\%) & 71.5\% & (0.8\%)
                & 12.1\% & (0.3\%) & 75.0\% & (0.9\%)
                \\
            \midrule
            \textbf{\textsc{Syntax}}
                & \bf 14.1\% & \bf (0.3\%) & 73.1\% & (1.0\%)
                & \bf 11.1\% & \bf (0.2\%) & 76.1\% & (1.0\%)
                \\
            \bottomrule
        \end{tabular}}%
    \vspace{-11.4pt}
\end{table}

\begin{figure}[t]
    \centering
    \begin{subfigure}{\linewidth}
        \centering
        \includegraphics[width=.5\linewidth]{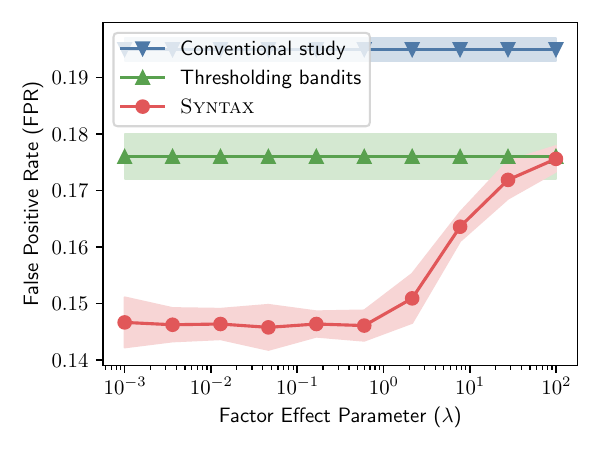}%
        \includegraphics[width=.5\linewidth]{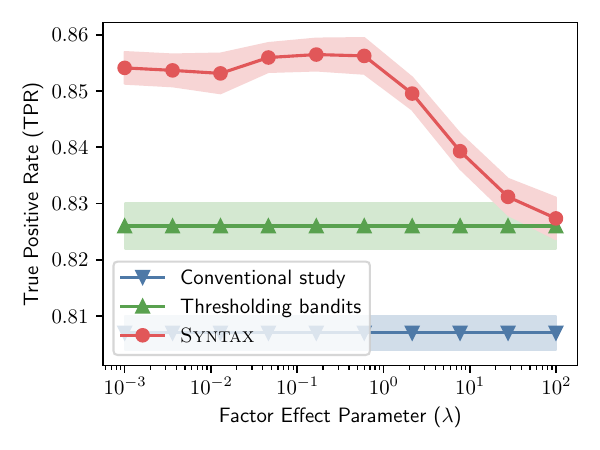}%
        \vspace{-6pt}%
        \caption{\textit{Diminishing Factor Effects}, $H=200$}%
    \end{subfigure}%
    
    \vspace{3pt}%
    \begin{subfigure}{\linewidth}
        \centering
        \includegraphics[width=.5\linewidth]{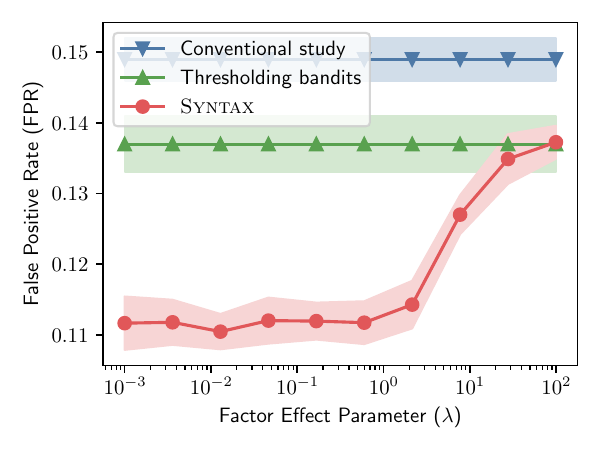}%
        \includegraphics[width=.5\linewidth]{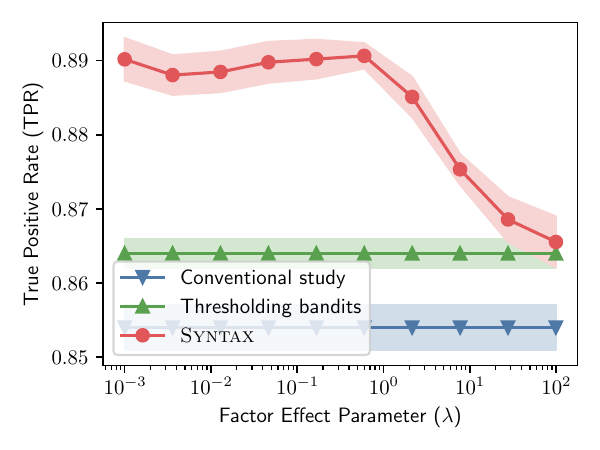}%
        \vspace{-6pt}%
        \caption{\textit{Diminishing Factor Effects}, $H=400$}%
    \end{subfigure}%
    \vspace{-3pt}%
    \caption{\textbf{Sensitivity of \textsc{Syntax} to parameter $\bm{\lambda}$.}}%
    \label{fig:sensitivity}%
    \vspace{-\baselineskip+3pt}
\end{figure}

For the first scenario, we consider the same simulation setup as \textit{Diminishing Factor Effects} but generate mean outcomes according to equation
\begin{align}
    \bar{y}_{it} = \delta_t + \bm{w}_t^{\mathsf{T}}(\bm{x}_i^2) + \bm{\mu}_t^{\mathsf{T}}\bm{z}_i \quad\text{for}\quad t\in[T]
\end{align}
instead, where $\bm{x}_i^2$ denotes the element-wise square of $\bm{x}_i$. The results are given in Table~\ref{tab:results-additional-a}. We see that all methods based on synthetic control, including \textsc{Syntax}, naturally lose performance after this change. However, \textsc{Syntax} still outperforms all benchmarks.

For the second scenario, we again consider the same simulation setup as \textit{Diminishing Factor Effects} but this time set $D_z=T$. The results are given in Table~\ref{tab:results-additional-b}. We see that TPR of \textsc{Syntax} is more sensitive to this change than its FPR. This observation is directly related to a balance between FPR-TPR. In all of our experiments, we fixed the threshold for declaring a positive effect at zero (for instance, see line 11 in Algorithm~\ref{alg:syntax}). Adjusting this threshold would have tuned the balance between FPR nad TPR (higher thresholds would achieve better FPR but worse TPR). Fixing the threshold at zero is a natural choice when all of our benchmarks are based on unbiased estimates. Under model mismatch however, synthetic inference becomes biased, the effect of which is equivalent to changing the classification threshold. Since the FPR performance stays better relative to the TPR performance, we can tell that the bias introduced happens to be negative (equivalent to a higher classification threshold).

\paragraph{Sensitivity to the Factor  Effect Parameter}
Tuning hyper-parameters is a general challenge that affects all online algorithms, including \textsc{Syntax}, since in an online setting, no a-priori data would be available to perform cross-validation. We run additional experiments to evaluate the sensitivity of \textsc{Syntax} to its hyper-parameter $\lambda$ for \textit{Diminishing Factor Effects} by varying $\lambda$ logarithmically from $0.001$ to $100$ (the ideal $\lambda$ according to Proposition~\ref{prop:main} varies between $0.1$ and~$1$). The results are given in Figure~\ref{fig:sensitivity}. We see that \textsc{Syntax} outperforms all of the benchmarks without the hyper-parameter $\lambda$---namely \textit{Conventional Study} and \textit{Thresholding Bandits}---for almost all configurations. Notably, the performance of \textsc{Syntax} degrades only when $\lambda$ is exceedingly large and converges to that of \textit{Thresholding Bandits}. This is because $\lambda$ punishes the errors in synthetic representations ($\|\bm{\beta}-\bm{1}_i\|$). As $\lambda$ gets larger, \textsc{Syntax} reverts back to representing each subpopulation only as itself ($\bm{\beta}=\bm{1}_i$) and hence becomes equivalent to \textit{Thresholding Bandits}.

%%%
\vspace{-3pt}
\section{CONCLUSION}
\vspace{-3pt}

We introduced \textsc{Syntax}, a clinical trial design that recruits patients adaptively to identify subpopulations with positive treatment effect. As we already argued for in the introduction, running trials that take multiple populations into consideration is absolutely essential. By reducing the length of such trials through designs like \textsc{Syntax} that seek to make more efficient use of sample would not only make multi-population trials more feasible but also help treatment get to the market sooner, benefiting the patients. Although we presented the problem mainly from a clinical trial perspective, \textsc{Syntax} can generally be applied to any thresholding bandit problem where time-series context (such as baseline responses) are available.

\paragraph{On Clinical Equipoise}
From the perspective of clinical equipoise, the use of adaptive trials, including \textsc{Syntax}, is understood to be situational \citep{palmer1999ethics,fillion2019clinical}. We emphasize that running a \textsc{Syntax}-based clinical trial should only be considered if there is a genuine lack of information to determine the target population for a promising new treatment, especially when responses of different subpopulations to the treatment are expected to be highly heterogenous.
More specifically, the principle of clinical equipoise requires two conditions to be met: (i) there should be uncertainty among clinicians regarding the effectiveness of a treatment, and (ii) the results of a clinical trial should be convincing enough to resolve this uncertainty \citep{freedman1987equipose,miller2007clinical}. We examine \textsc{Syntax} on these tow axises:
\begin{enumerate}[label=(\roman*),nosep,parsep=\parskip,leftmargin=18pt]
    \item \justify{\textsc{Syntax} is an exploratory trial design rather than a confirmatory trial design. This means that, even at the end of a \textsc{Syntax}-based trial, there would still be genuine uncertainty regarding the effectiveness of the treatment for all subpopulations (hence the need for a subsequent confirmatory trial). At no point during the trial, \textsc{Syntax} assigns a patient to a treatment that can readily be confirmed to be ineffective for them; and the first condition is maintained through the trial. Moreover, the ultimate goal of \textsc{Syntax} is to identify subpopulations as the potential targets of a subsequent Phase III trial. As such, it would be in contradiction with the first condition to use \textsc{Syntax} when there is already a clear candidate to target in such a trial. As we already mentioned above, the use of \textsc{Syntax} should be reserved to the cases where there is disagreement in terms of which subpopulations would be the most suitable targets of a Phase III trial.}
    
    \item \justify{Being an exploratory trial design, \textsc{Syntax} cannot resolve all uncertainties regarding the effectiveness of a treatment. But, it can determine the appropriate target population for a subsequent confirmatory trial, and thereby help satisfy the second condition. As pointed out by \citet{chiu2018design}, a confirmatory trial is more likely to be inconclusive regarding subpopulation effects when the potential heterogeneity in patient responses is not taken into account.}
\end{enumerate}

\justify{\paragraph{On Real-data Validation}
A real-data validation of \textsc{Syntax} would essentially require running a new clinical trial. This would be infeasible (and potentially unethical) for a method development paper as ours. Hence, the standard approach to evaluating adaptive clinical trial designs is only to use simulated data \citep[e.g.][]{onur2019sequential,curth2023adaptively,huyuk2023when}, even in the biostatistics literature \citep[e.g.][]{friede2012conditional,magnusson2013group,stallard2014adaptive,henning2015closed,rosenblum2016group}.}

% \clearpage
\bibliographystyle{myIEEEtranSN}
\bibliography{references}

\section*{Checklist}
\begin{enumerate}

 \item For all models and algorithms presented, check if you include:
 \begin{enumerate}
   \item A clear description of the mathematical setting, assumptions, algorithm, and/or model. [Yes] See Section~\ref{sec:formulation} for the mathematical setting and assumptions, see \eqref{eqn:synthetic} and \eqref{eqn:_proofa} in Section~\ref{sec:formulation} for our model, and see Algorithm~\ref{alg:syntax} in Section~\ref{sec:syntax} for our algorithm.
   \item An analysis of the properties and complexity (time, space, sample size) of any algorithm. [Not Applicable]
   \item (Optional) Anonymized source code, with specification of all dependencies, including external libraries. [Yes]
 \end{enumerate}

 \item For any theoretical claim, check if you include:
 \begin{enumerate}
   \item Statements of the full set of assumptions of all theoretical results. [Yes] See Section~\ref{sec:formulation} for the full set of assumptions.
   \item Complete proofs of all theoretical results. Yes, proofs are given in Appendix~\ref{sec:proofs-appendix}.
   \item Clear explanations of any assumptions. [Yes] See ``On Modeling Assumptions'' in Section~\ref{sec:formulation}.
 \end{enumerate}

 \item For all figures and tables that present empirical results, check if you include:
 \begin{enumerate}
   \item The code, data, and instructions needed to reproduce the main experimental results (either in the supplemental material or as a URL). [Yes] See ``Environments'' and ``Benchmarks'' in Section~\ref{sec:experiments} and the benchmark algorithms in Appendix~\ref{sec:benchmark-appendix}.
   \item All the training details (e.g., data splits, hyperparameters, how they were chosen). [Yes] See Appendix~\ref{sec:benchmark-appendix}.
   \item A clear definition of the specific measure or statistics and error bars (e.g., with respect to the random seed after running experiments multiple times). [Yes] See ``Environments'' in Section~\ref{sec:experiments}.
   \item A description of the computing infrastructure used. (e.g., type of GPUs, internal cluster, or cloud provider). [Yes] See Appendix~\ref{sec:benchmark-appendix}.
 \end{enumerate}

 \item If you are using existing assets (e.g., code, data, models) or curating/releasing new assets, check if you include:
 \begin{enumerate}
   \item Citations of the creator If your work uses existing assets. [Not Applicable]
   \item The license information of the assets, if applicable. [Not Applicable]
   \item New assets either in the supplemental material or as a URL, if applicable. [Not Applicable]
   \item Information about consent from data providers/curators. [Not Applicable]
   \item Discussion of sensible content if applicable, e.g., personally identifiable information or offensive content. [Not Applicable]
 \end{enumerate}

 \item If you used crowdsourcing or conducted research with human subjects, check if you include:
 \begin{enumerate}
   \item The full text of instructions given to participants and screenshots. [Not Applicable]
   \item Descriptions of potential participant risks, with links to Institutional Review Board (IRB) approvals if applicable. [Not Applicable]
   \item The estimated hourly wage paid to participants and the total amount spent on participant compensation. [Not Applicable]
 \end{enumerate}

 \end{enumerate}

\newpage
\appendix
\onecolumn

\section{PROOFS OF PROPOSITIONS}
\label{sec:proofs-appendix}

\subsection{Proof of Proposition~\ref{prop:main}}

Denote with $e_{iT}^{\mathsmaller{(0)}}=\hat{y}_{iT}^{\mathsmaller{(0)}}-\bar{y}_{iT}$, $e_{iT}^{\mathsmaller{(1)}}=\hat{y}_{iT}^{\mathsmaller{(1)}}-\bar{y}_{iT}-r_i$, and $e_{it}=\hat{y}_{it}-\bar{y}_{it}$ the observation noises.
Similar to \citet{abadie2010synthetic}, we start by showing that unobservable factor loadings $\bm{z}_i$ can be inferred through observable responses $\hat{y}$ if responses are observed for a long enough pre-treatment period. More specifically, a good match in terms of responses $\smash{\bm{\hat{y}}_{i\neg T}=\hat{Y}_{\neg T}\bm{\beta}}$ leads to a good match in terms of factor loadings $\bm{z}_i\approx Z\bm{\beta}$. We have
\begin{align}
    0 &= \bm{\hat{y}}_{i\neg T}-\hat{Y}_{\neg T}\bm{\beta} \\
    &= \bm{e}_{i\neg T}-E_{\neg T}\bm{\beta}+\bm{\bar{y}}_{i\neg T}-\bar{Y}_{\neg T}\bm{\beta} \label{eqn:proofa} \\
    &= \bm{e}_{i\neg T}-E_{\neg T}\bm{\beta}
        + \bm{\delta}_{\neg T} + W_{\neg T}^{\mathsf{T}}\bm{x}_i + M_{\neg T}^{\mathsf{T}}\bm{z}_i 
        - (\bm{\delta}_{\neg T}\bm{1}^{\mathsf{T}}+W_{\neg T}^{\mathsf{T}}X+M_{\neg T}^{\mathsf{T}}Z)\bm{\beta} \label{eqn:proofb} \\
    &= \bm{e}_{i\neg T}-E_{\neg T}\bm{\beta} + M_{\neg T}^{\mathsf{T}}(\bm{z}_i-Z\bm{\beta}) \label{eqn:proofc}
\end{align}
where \eqref{eqn:proofa} is due to \eqref{eqn:_proofa}, \eqref{eqn:proofb} is due to \eqref{eqn:synthetic}, and \eqref{eqn:proofc} holds if $\bm{x}_i=X\bm{\beta}$ and $\bm{1}^{\mathsf{T}}\bm{\beta}=1$. Hence,
\begin{align}
    \bm{z}_i-Z\bm{\beta} &= (M^{+}_{\neg T})^{\mathsf{T}}E_{\neg T}(\bm{\beta}-\bm{1}_i) \label{eqn:proofd}
\end{align}
\justify{where $M_{\neg T}^{+}=M_{\neg T}^{\mathsf{T}}(M_{\neg T}M_{\neg T}^{\mathsf{T}})^{-1}$ is the right inverse of $M_{\neg T}$. Notably, this right inverse exists only if $M_{\neg T}$ has full rank and $T>D_z$ (i.e.\ when the pre-treatment period is long enough). It is already possible to spot that the term in \eqref{eqn:proofd} is the source of representation error in Proposition~\ref{prop:main}. This provides the insight that representation error is essentially caused by the imperfect estimation of factor loadings $\bm{z}_i$ through noisy observations of responses $\hat{y}$.}

With this result, we are now ready to characterize the estimation error $r_i-\hat{r}_i(\beta)$ in terms of observation noises $e_{it},e^{\mathsmaller{(0)}}_{iT},e^{\mathsmaller{(1)}}_{iT}$. We have
\begin{align}
    &r_i - \hat{r}_i(\bm{\beta}) \nonumber \\
    &\quad= r_i-\hat{y}_{iT}^{\mathsmaller{(1)}}+\bm{\beta}^{\mathsf{T}}\bm{\hat{y}}_{\cdot T}^{\mathsmaller{\mathsmaller{(0)}}} \\
    &\quad= r_i -(r_i+\bar{y}_{iT}+e_{iT}^{\mathsmaller{(1)}})+\bm{\beta}^{\mathsf{T}}(\bm{\bar{y}}_{\cdot T}+\bm{e}_{\cdot T}^{\mathsmaller{(0)}}) \label{eqn:proofe} \\
    &\quad = -e_{iT}^{\mathsmaller{(1)}}+\bm{\beta}^\mathsf{T}\bm{e}_{\cdot T}^{\mathsmaller{(0)}}
        - \delta_T-\bm{w}_T^{\mathsf{T}}\bm{x}_i-\bm{\mu}_T^{\mathsf{T}}\bm{z}_i
        + \bm{\beta}^{\mathsf{T}}(\delta_T\bm{1}+X^{\mathsf{T}}\bm{w}_T+Z^{\mathsf{T}}\bm{\mu}_T) \label{eqn:prooff} \\
    &= -e_{iT}^{\mathsmaller{(1)}}+\bm{\beta}^\mathsf{T}\bm{e}_{\cdot T}^{\mathsmaller{(0)}} + \bm{\mu}_T^{\mathsf{T}}(\bm{z}_i-Z\bm{\beta}) \label{eqn:proofg} \\
    & = -e_{iT}^{\mathsmaller{(1)}}+\bm{\beta}^\mathsf{T}\bm{e}_{\cdot T}^{\mathsmaller{(0)}} + \bm{\mu}_T^{\mathsf{T}}(M^{+}_{\neg T})^{\mathsf{T}}E_{\neg T}(\bm{\beta}-\bm{1}_i)
\end{align}
where \eqref{eqn:proofe} is due to \eqref{eqn:_proofe}, \eqref{eqn:prooff} is due to \eqref{eqn:synthetic}, and \eqref{eqn:proofg} holds if $\bm{x}_i=X\bm{\beta}$ and $\bm{1}^{\mathsf{T}}\bm{\beta}=1$.
Since $e_{iT}^{\mathsmaller{(0)}}$, $e_{iT}^{\mathsmaller{(1)}}$, and $\{e_{it}\}_{t<T}$ all have zero mean and are independent from each other, $\mathbb{E}[r_i-\hat{r}_i(\bm{\beta})]=0$ and
\begin{align}
    \mathbb{V}[r_i-\hat{r}_i(\bm{\beta})] &= \mathbb{E}[(r_i-\hat{r}_i(\bm{\beta}))^2] \\
    &= \mathbb{E}[(e_{iT}^{\mathsmaller{(1)}})^2] + \mathbb{E}[\|\bm{\beta}\|_{\bm{e}_{\cdot T}^{\mathsmaller{(0)}}(\bm{e}_{\cdot T}^{\mathsmaller{(0)}})^{\mathsf{T}}}^2]
        + \mathbb{E}[\|\bm{\mu}_T^{\mathsf{T}}(M^{+}_{\neg T})^{\mathsf{T}}E_{\neg T}(\bm{\beta}-\bm{1}_i)\|^2] \\
    &\leq \mathbb{E}[(e_{iT}^{\mathsmaller{(1)}})^2] + \mathbb{E}[\|\bm{\beta}\|_{\bm{e}_{\cdot T}^{\mathsmaller{(0)}}(\bm{e}_{\cdot T}^{\mathsmaller{(0)}})^{\mathsf{T}}}^2]
        + \|M^{+}_{\neg T}\bm{\mu}_T\|^2\mathbb{E}[\|\bm{\beta}-\bm{1}_i\|^2_{E_{\neg T}^{\mathsf{T}}E_{\neg T}}] \\
    &= \sigma^2/n_i^{\mathsmaller{(1)}} +\sigma^2\|\bm{\beta}\|_{(N^{\mathsmaller{(0)}})^{-1}}^2 + \lambda\sigma^2\|\bm{\beta}-\bm{1}_i\|_{N^{-1}}^2
\end{align}

\subsection{Proof of Proposition~\ref{prop:corollary}}

Proposition~\ref{prop:corollary} is a corollary of Proposition~\ref{prop:main}. We simply have
\begin{align}
    \mathbb{V}[r_i-\hat{r}_i(\bm{\beta}^*_i)] &\leq V_i(\bm{\beta}^*_i) \label{eqn:__proofa} \\
    &\leq V_i(\bm{1}_i) \label{eqn:__proofb} \\ 
    &= \mathbb{V}[r_i-\hat{r}_i^{~\text{naive}}] 
\end{align}
where \eqref{eqn:__proofa} is due to Proposition~\ref{prop:main} and \eqref{eqn:__proofb} is by definition of $\bm{\beta}_i^*$ in \eqref{eqn:_proofj}.

\section{BENCHMARK ALGORITHMS}
\label{sec:benchmark-appendix}

Algorithm~\ref{alg:benchmark} summarizes all benchmarks (except \textsc{Syntax}, which is given in Algorithm~\ref{alg:syntax} instead). We set the factor effect parameter~$\lambda$ to its optimal value given in Proposition~\ref{prop:main} (except for the sensitivity experiments in Figure~\ref{fig:sensitivity}). This means the results we present are for perfectly tuned versions of each algorithm. All experiments are run on a personal computer with an Intel i9 processor. Finally, the code to reproduce our experimental results can be found at \href{https://github.com/alihanhyk/syntax}{\texttt{https://github.com/alihanhyk/syntax}} and \href{https://github.com/vanderschaarlab/syntax}{\texttt{https://github.com/vanderschaarlab/syntax}}.

\begin{algorithm}[H]
    \caption{Benchmarks}
    \label{alg:benchmark}
    \vspace{3pt}
    \textbf{Parameters:} Episode horizon~$H$, factor effect parameter~$\lambda$ \\
    \textbf{Output:} Subpopulations~$\hat{\mathcal{I}}^*$ with positive treatment effect
    \vspace{3pt}
    \hrule
    \begin{algorithmic}[1]
        \STATE $n_i^{\mathsmaller{(\alpha)}}\gets 0$,~~$\hat{y}_{iT}^{\mathsmaller{(\alpha)}}\gets 0$,~~$\bm{\hat{y}}_{i\neg T}\gets\bm{0}$,~~$\forall i\in[K],\alpha\in\{0,1\}$
        \FOR{$\eta\in\{1,2,\ldots,H\}$}
            \IF{\textit{Conventional Study} \textbf{or} \textit{Synthetic Study}}
                \STATE Sample $i,\alpha$ uniformly at random from $[K]\times\{0,1\}$
            \ELSIF{\textit{Thresholding Bandits}}
                \STATE $i\gets \argmin_{i\in[K]}$ $S_i^{~\text{naive}}=|\hat{r}_i^{~\text{naive}}|/(1/n_i^{\mathsmaller{(0)}}+1/n_i^{\mathsmaller{(1)}})^{1/2}$
                \STATE $\alpha\gets \argmin_{\alpha\in\{0,1\}} n_i^{\mathsmaller{(\alpha)}}$
            \ELSIF{\textit{Synthetic Design}}
                \STATE $i,\alpha \gets \argmin_{i\in[K],\alpha\in\{0,1\}} \max_{i^*\in[K]} \min_{\bm{\beta}:\: \bm{x}_{i^*}=X\bm{\beta},\, \bm{\hat{y}}_{i^*\neg T}=\hat{Y}_{\neg T}\bm{\beta},\,\bm{1}^{\mathsf{T}}\bm{\beta}=1}$ \\
                \hfill $V_{i^*}(\bm{\beta}; N^{\mathsmaller{(0)}}+(1-\alpha)\bm{1}_i{\bm{1}_i}^{\mathsf{T}}\!, N^{\mathsmaller{(1)}}+\alpha\bm{1}_i{\bm{1}_i}^{\mathsf{T}})$
            \ENDIF
            \STATE Recruit from population~$i$ and treatment group~$\alpha$
            \STATE Observe pre-treatment outcomes $\bm{y}_{\neg T}=[y_1\cdots y_{T-1}]^{\mathsf{T}}$ and the final outcome $y_T$
            \STATE $n_i^{\mathsmaller{(\alpha)}}\gets n_i^{\mathsmaller{(\alpha)}}+1$
            \STATE $\hat{y}_{iT}^{\mathsmaller{(\alpha)}}\gets \hat{y}_{iT}^{\mathsmaller{(\alpha)}} + (y_T-\hat{y}_{iT}^{\mathsmaller{(\alpha)}})/n_i^{\mathsmaller{(\alpha)}}$
            \STATE $\bm{\hat{y}}_{i\neg T}\gets \bm{\hat{y}}_{i\neg T} + (\bm{y}_{\neg T}-\bm{\hat{y}}_{i\neg T})/n_i$
        \ENDFOR
        \IF{\textit{Conventional Study} \textbf{or} \textit{Thresholding Bandits}}
            \STATE $\hat{\mathcal{I}}^*\gets \{i\in[K]: \hat{r}_i^{~\text{naive}} > 0\}$
        \ELSIF{\textit{Synthetic Study} \textbf{or} \textit{Synthetic Design}}
            \STATE $\hat{\mathcal{I}}^*\gets \{i\in[K]: \hat{r}_i(\bm{\beta}^*_i) > 0\}$
        \ENDIF
    \end{algorithmic}
\end{algorithm}

\end{document}